\documentclass[sigconf, nonacm]{acmart}
\renewcommand\footnotetextcopyrightpermission[1]{}
\usepackage{algorithm}
\usepackage{algpseudocode}
\usepackage{amsmath}
\usepackage{mathtools}
\usepackage{pifont}
\usepackage{bm}
\usepackage{microtype}
\usepackage{graphicx}
\usepackage{booktabs}  
\usepackage{multirow}
\usepackage{xcolor}     
\usepackage{newunicodechar}
\newunicodechar{❶}{\ding{172}}
\newunicodechar{❷}{\ding{173}}
\usepackage{colortbl}
\usepackage{bbding}
\usepackage{wrapfig}
\usepackage[capitalize,noabbrev]{cleveref}
\definecolor{mygreen}{RGB}{46, 139, 87}  
\definecolor{myred}{RGB}{205, 92, 92}     
\definecolor{mygray}{gray}{0.6}  
\definecolor{lightyellow}{RGB}{255, 250, 205} 
\definecolor{algcomment}{RGB}{70,130,180}
\newcommand{\conf}[1]{\textcolor{mygray}{\scriptsize #1}}
\usepackage{amsthm}
\newtheorem{theorem}{Theorem}
\newtheorem{lemma}{Lemma}
\newtheorem{corollary}{Corollary}

\newcommand{\myconf}[1]{\textcolor{mygray}{\scriptsize #1}}
\newcommand{\up}[1]{\ensuremath{_{\textcolor{mygreen}{\uparrow #1}}}}
\newcommand{\down}[1]{\ensuremath{_{\textcolor{myred}{\downarrow #1}}}}
\newcommand{\second}[1]{\underline{#1}}
\newcommand{\best}[1]{\textbf{#1}}
\AtBeginDocument{%
  }

\setcopyright{acmlicensed}
\copyrightyear{2018}
\acmYear{2018}
\acmDOI{XXXXXXX.XXXXXXX}
\acmConference[Conference acronym 'XX]{Make sure to enter the correct
  conference title from your rights confirmation email}{June 03--05,
  2018}{Woodstock, NY}
\acmISBN{978-1-4503-XXXX-X/2018/06}
\begin{document}

\title{Structural Entropy-Driven Graph Diffusion Generation for One-Shot Federated
Graph Learning}


%% ==================== 作者与单位信息 ====================
% 1. 共一第一顺位
\author{Shutong Zheng}
\authornote{Both authors contributed equally to this research.}
\affiliation{%
  \institution{Sun Yat-sen University}
  \city{Guangzhou}
  \country{China}
}
\email{zhengsht29@mail2.sysu.edu.cn}

% 2. 共一第二顺位
\author{Lele Fu}
\authornotemark[1]
\affiliation{%
  \institution{Sun Yat-sen University}
  \city{Guangzhou}
  \country{China}
}
\email{fulle@mail2.sysu.edu.cn}

% 3. 二作
\author{Sheng Huang}
\affiliation{%
  \institution{Sun Yat-sen University}
  \city{Guangzhou}
  \country{China}
}
\email{huangsh253@mail2.sysu.edu.cn}

% 4. 三作
\author{Wei Yang Bryan Lim}
\affiliation{%
  \institution{Nanyang Technological University}
  \country{Singapore}
}
\email{bryan.limwy@ntu.edu.sg}

% 5. 通讯作者
\author{Chuan Chen}
\authornote{Corresponding author.}
\affiliation{%
  \institution{Sun Yat-sen University}
  \city{Guangzhou}
  \country{China}
}
\email{chenchuan@mail.sysu.edu.cn}

% 页眉短作者（第一作者 et al.）
\renewcommand{\shortauthors}{Zheng and Fu, et al.}
%% ========================================================

\begin{abstract}
One-shot federated graph learning (FGL) requires the server to estimate client contributions from highly compressed information, yet conventional volume-based weighting captures the amount of client data while overlooking how its connectivity is organized. In this paper, we propose SPIRE, a Structural Entropy-Driven Graph Diffusion Generation method that introduces topology-aware client differentiation into one-shot FGL. Specifically, we employ first-order degree-distribution structural entropy as a compact descriptor of degree-mass dispersion and use it to derive structural client weights, providing an inductive bias that accounts for differences in graph topology beyond data volume. On the generation side, a graph diffusion model on the server synthesizes pseudographs conditioned on the weighted client prototypes, capturing both semantic and structural information without requiring additional client-side training. The generated pseudographs are then assembled via disjoint union fusion to train a global graph neural network. Extensive experiments on seven real-world graph datasets demonstrate that SPIRE consistently outperforms conventional and one-shot FGL methods, with particularly strong gains under highly heterogeneous (non-IID) and graph-perturbed settings.
\end{abstract}



\maketitle

\section{Introduction}

Federated learning (FL) enables collaborative model training without centralizing raw data. Extending this paradigm to graph-structured data has spurred Federated Graph Learning (FGL). However, unlike conventional FL with independent samples, client graphs are inherently coupled with topology. They often differ substantially in both data distribution and structural organization, making effective knowledge aggregation exceptionally challenging.

To address this heterogeneity, advanced FGL methods (e.g., FedSage~\cite{zhang2021subgraph}, FedTAD~\cite{zhu2024fedtad}, FedATH~\cite{fu2025less}) have been proposed. While effective in improving cross-client alignment, they rely on iterative communication. In practical large-scale web and social data mining, such as collaborative recommendation systems across regional data centers or decentralized social networks, repeatedly synchronizing topological embeddings across isolated data silos incurs prohibitive bandwidth costs and exacerbates privacy risks. This operational bottleneck motivates one-shot FGL, which completes collaborative graph mining in a single communication round. Under this strict constraint, the server must aggregate information from highly compressed statistics, making the determination of actual client influence a central challenge.

\begin{figure}[t]
\centering
\includegraphics[width=0.48\textwidth]{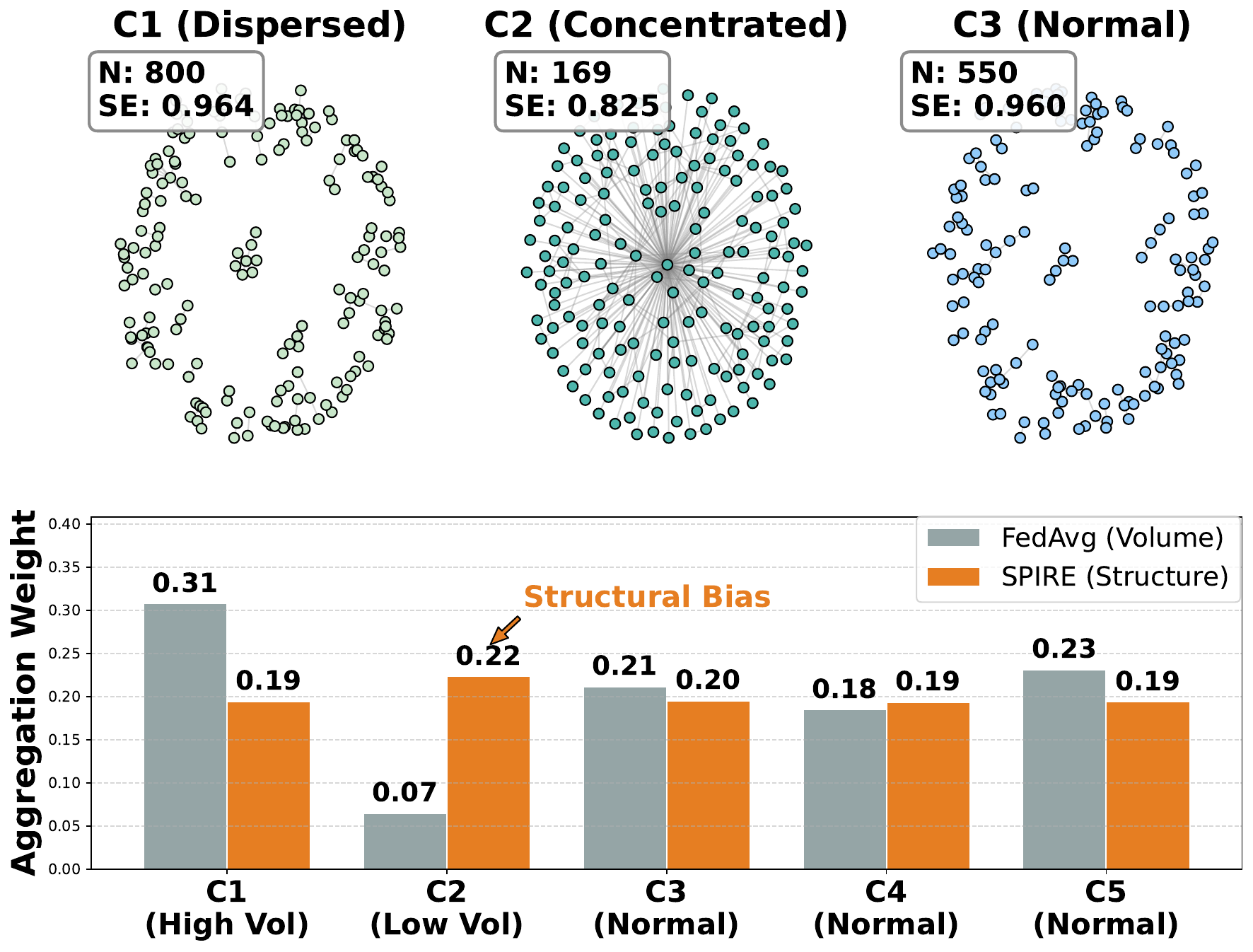}
\caption{
Illustration of the volume--structure mismatch in client weighting.
C1 is a large ``Noisy Giant'' (800 nodes), while C2 is a smaller ``Clean Expert'' (169 nodes) with differentiated community structure; C3--C5 follow standard non-IID partitions.
FedAvg favors C1 by volume, whereas SPIRE differentiates client influence using structural entropy.
}
\label{fig:weight_reversal}
\end{figure}

Existing one-shot FGL methods, such as GHOST~\cite{qian2025ghost} and OASIS~\cite{wanoasis}, typically determine client influence through data volume or uniform aggregation, leaving the structural characteristics of local graphs largely unaccounted for. Such volume-based aggregation implicitly assumes that larger clients should exert greater influence, regardless of how their connectivity is organized. This creates a fundamental volume--structure mismatch: client volume reflects how much data a client possesses, but does not characterize how that data is structurally organized. As illustrated in Fig.~\ref{fig:weight_reversal}, FedAvg assigns greater influence to C1 simply because it contains more nodes, despite the markedly different connectivity distributions of C1 and C2. This observation motivates topology-aware client differentiation: rather than relying solely on graph volume, the server should incorporate structural information when determining client influence. However, the one-shot constraint makes it impractical to rely on iterative structural alignment or direct access to client topologies, calling for a compact structural descriptor that can be obtained from limited client-side information.

To address this challenge, we propose \textbf{SPIRE}, a \textbf{S}tructural Entro\textbf{p}y-Dr\textbf{i}ven G\textbf{r}aph Diffusion G\textbf{e}neration method for one-shot FGL. Unlike node count, which reflects only graph volume, first-order structural entropy ($SE^{(1)}$) captures the underlying distribution of connectivity. We formulate this compact descriptor into an entropy-regularized influence mechanism. Rather than stopping at initial aggregation, SPIRE propagates this topology-aware bias to explicitly govern four downstream stages: semantic prototype aggregation, Wasserstein barycentric alignment, generation budget allocation, and pseudograph assembly. Guided by this influence, a conditional diffusion model synthesizes node features, which are then instantiated into topology via k-NN. By shifting generative computation to the server and bypassing dense adjacency optimization, SPIRE ensures a lightweight communication protocol with favorable large-graph scalability.

Overall, our main contributions are summarized as follows:
\begin{itemize}
    \item \textbf{Problem Insight:} We identify a fundamental volume--structure mismatch in one-shot FGL, where client data volume alone fails to reflect actual topological utility, motivating structure-aware client differentiation.
    
    \item \textbf{Topology-Aware Influence Allocation:} We propose an entropy-regularized influence mechanism that consistently propagates structural bias to govern downstream global reconstruction, rather than serving merely as a static aggregation weight.
    
    \item \textbf{Structure-Guided Reconstruction:} We develop a feature-first pseudograph synthesis via conditional diffusion and k-NN assembly, enabling scalable global GNN training while avoiding direct dense-adjacency optimization.
    
    \item \textbf{Theoretical and Empirical Validation:} We derive rigorous sensitivity bounds for entropy perturbations. Evaluations across seven datasets confirm SPIRE's state-of-the-art performance, robustness to severe heterogeneity, and practical large-scale scalability.
\end{itemize}

\section{Related Work}
\subsection{Conventional Federated Learning}
Conventional federated learning enables collaborative model
training across distributed clients without sharing raw
data~\cite{fu2025beyond,mcmahan2017communication}. While
achieving strong performance on IID data, conventional FL
suffers from statistical heterogeneity
and communication inefficiency under non-IID
settings~\cite{li2020fedprox,karimireddy2020scaffold,
li2021moon,li2024cooperative,chen2024mixed}. Recent
studies improve robustness through knowledge
distillation~\cite{zhang2025model,zhang2024fedgmkd},
generative modeling~\cite{lai2025pfedgpa,chen2025fedbip},
and prototype-based aggregation
strategies~\cite{huang2023rethinking,zhang2024fedtgp}.
However, these approaches are primarily designed for
Euclidean data and are not well suited for graph-structured
scenarios with complex topological dependencies.

\subsection{Federated Graph Learning}

Federated graph learning (FGL) extends federated learning to graph
domains and is commonly divided into graph-level and node-level
settings~\cite{lei2023federated,qu2023semi,pan2024towards}. In graph-level
FGL, clients hold isolated graph instances, and methods such as
GCFL+~\cite{xie2021federated} and FedStar~\cite{tan2023federated} learn
shared representations across heterogeneous graphs. In node-level FGL,
clients own subgraphs of a global graph, where methods including
FedBG~\cite{huang2025fedbg} and FGGP~\cite{wan2024federated} focus on
preserving cross-client structural dependencies. Recent studies further
address topological heterogeneity through topology-aware distillation
and causal decoupling, such as FedTAD~\cite{zhu2024fedtad} and
FedATH~\cite{fu2025less}.

Recent work like SEFGL~\cite{SEFGL} uses structural entropy for local graph optimization and iterative personalized prototype clustering in multi-round settings. However, under strict one-shot communication where iterative alignment is prohibited, utilizing structural descriptors shifts from localized similarity metrics to global influence allocation signals for central reconstruction.

\subsection{One-Shot Federated Learning}

One-shot federated learning reduces communication to a
single round and has recently attracted increasing
attention in bandwidth-limited scenarios ~\cite{allouah2024revisiting,zeng2024one,
tang2024fusefl,yang2024feddeo}. Existing methods can be
broadly categorized into:
(1) public-data distillation approaches,
(2) data-free generator-based methods,
and (3) synthetic-data sharing strategies.

Early methods~\cite{li2020practical,
salehkaleybar2021one,guha2019one} rely on public datasets
to distill ensemble knowledge from local models. To remove
the dependence on public data,
DENSE~\cite{zhang2022dense} trains a generator over client
model ensembles, while
FedSD2C~\cite{zhang2024one} enables clients to share
synthetic data instead of inconsistent local parameters.
One-shot federated graph learning introduces additional challenges due to the
strong coupling between node attributes and graph
topology. In the graph domain,
GHOST~\cite{qian2025ghost} introduces client-specific
proxy models to preserve topology-aware parameters during
aggregation, while
OASIS~\cite{wanoasis} constructs a structural latent space
via a topological codebook for synthetic graph generation.

However, existing one-shot FGL methods generally treat
client contributions uniformly or estimate them solely
based on data scale, without explicitly evaluating the
structural information of local graphs. Moreover, current
pseudograph generation methods often fail to preserve
meaningful topological semantics, limiting the reliability
of synthetic graph supervision.

\section{Preliminaries and Motivation}
\subsection{Background}

Consider an undirected graph $\mathcal{G} = (\mathcal{V}, \mathcal{E}, \mathbf{X}, \mathbf{Y})$, where $\mathcal{V}$ is the set of $N$ nodes and $\mathcal{E}$ is the set of edges. 
Each node $v_i \in \mathcal{V}$ is associated with a feature vector $\mathbf{x}_i \in \mathbb{R}^F$ and a corresponding label $\mathbf{y}_i$, where $\mathbf{X} \in \mathbb{R}^{N \times F}$ and $\mathbf{Y}$ denote the feature matrix and label matrix, respectively. 
The topological structure is represented by an adjacency matrix $\mathbf{A} \in \{0, 1\}^{N \times N}$, where $\mathbf{A}_{ij} = 1$ if $(v_i, v_j) \in \mathcal{E}$, and $0$ otherwise. 
Graph Neural Networks (GNNs) encode node embeddings by recursively aggregating information from neighbors. Formally, the embedding $\mathbf{h}_v^{(l)}$ of node $v$ at the $l$-th layer is updated by:
\begin{equation}
    \mathbf{h}_v^{(l+1)} = \text{UPD}\left(\mathbf{h}_v^{(l)}, \text{AGG}\left(\left\{\mathbf{h}_u^{(l)} : u \in \mathcal{N}(v)\right\}\right)\right),
\end{equation}
where $\mathcal{N}(v)$ is the neighbor set of node $v$, $\text{AGG}(\cdot)$ is an aggregation function, and $\text{UPD}(\cdot)$ denotes an update rule for node embeddings. Initially, for $l=0$, $\mathbf{h}_v^{(0)} = \mathbf{x}_v$.

Consider a federated system comprising a central server and a set of clients $\mathcal{C} = \{1, \dots, K\}$, where each client $k \in \mathcal{C}$ holds a private local graph $\mathcal{G}_k = (\mathcal{V}_k, \mathcal{E}_k, \mathbf{X}_k, \mathbf{Y}_k)$. In standard FGL, the server and clients engage in multiple communication rounds to iteratively optimize a global GNN. Specifically, the server aggregates local model parameters $\mathbf{W}_k$ to update the global model parameters $\mathbf{W}$ via:
\begin{equation}
\label{eq:volume}
    \mathbf{W} = \sum_{k=1}^{K}\frac{N_{k}}{N}\mathbf{W}_{k},
\end{equation}
where $N_{k} = |\mathcal{V}_k|$ is the node count of the $k$-th client, and $N=\sum_{k=1}^{K}N_{k}$ is the total number of nodes across all clients. 

\subsection{Motivation}

As shown in Eq.~\ref{eq:volume}, conventional aggregation explicitly relies on client data volume ($N_k$). While volume measures data size, graph utility is inherently coupled with topological structure. To decouple these factors, we construct four client archetypes on the Cora dataset by independently varying volume and degree-distribution entropy: Large/Concentrated, Small/Concentrated, Large/Dispersed, and Small/Dispersed (Figure~\ref{fig:motivation}).

The evaluation reveals a clear volume-structure mismatch: a large but structurally dispersed client frequently underperforms a small but concentrated client. Notably, while these isolated clients achieve high accuracies on their skewed local subgraphs, one-shot aggregation of such heterogeneous models causes severe weight divergence and global performance drops (evaluated in Section 6). This demonstrates that volume alone is insufficient to characterize a client's true structural contribution.

\begin{figure}[t]
  \centering
  \includegraphics[width=0.48\textwidth]{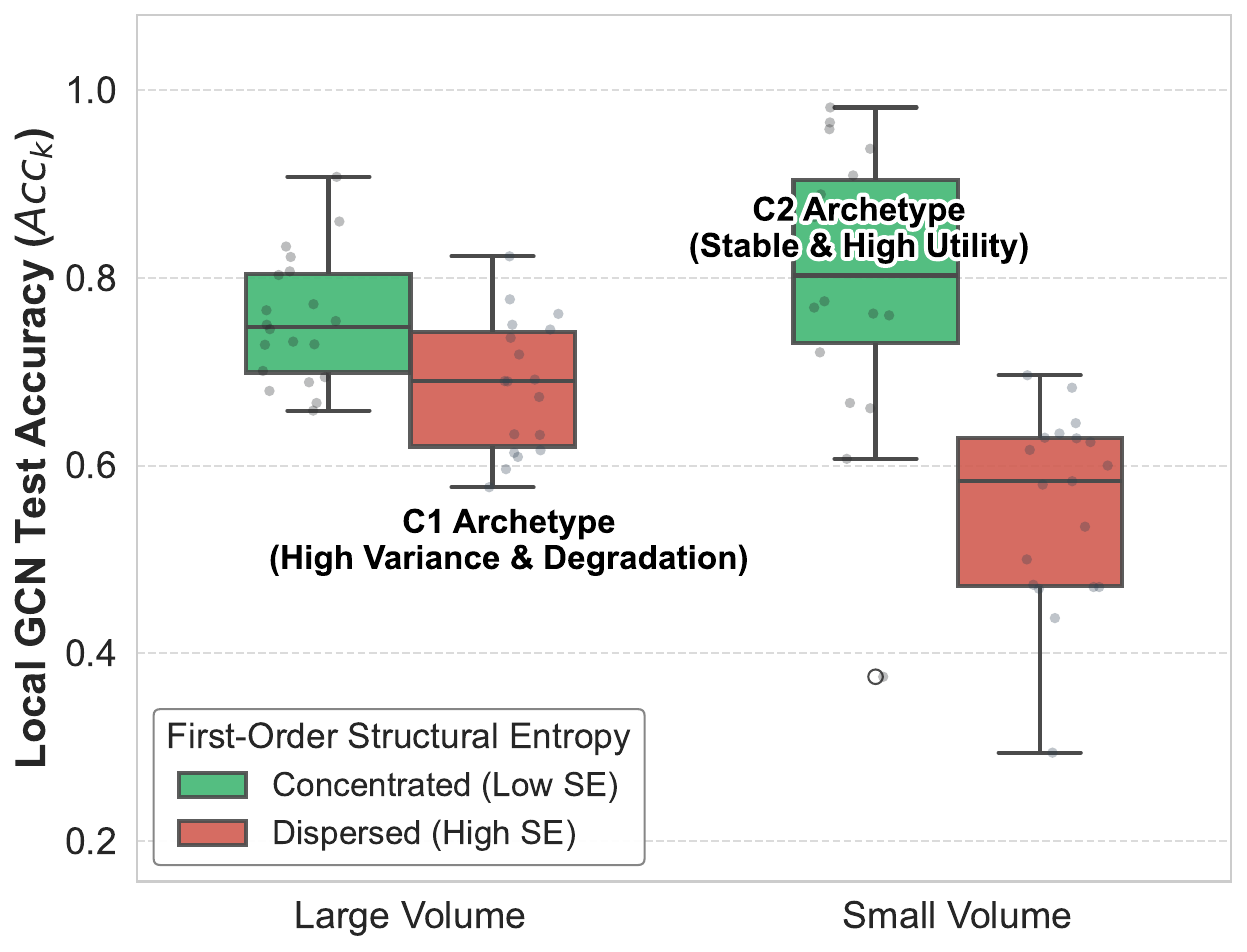}
  \vspace{-0.2in}
  \caption{
  Controlled decoupling of client volume and topology on Cora. Four client archetypes independently vary graph volume and degree-distribution entropy, revealing distinct local utilities.
  }
  \label{fig:motivation}
  \vspace{-0.15in}
\end{figure}

To capture this missing dimension, we employ \textit{first-order structural entropy} ($SE^{(1)}$) to summarize the concentration of the local degree-mass distribution. A lower $SE^{(1)}$ indicates a more concentrated topology. SPIRE thus utilizes $SE^{(1)}$ as a structural bias to differentiate client influence beyond data volume, without transmitting raw adjacency matrices.

\section{Methodology}

\begin{figure*}[t]
	\centering
	\includegraphics[width=0.98\textwidth]{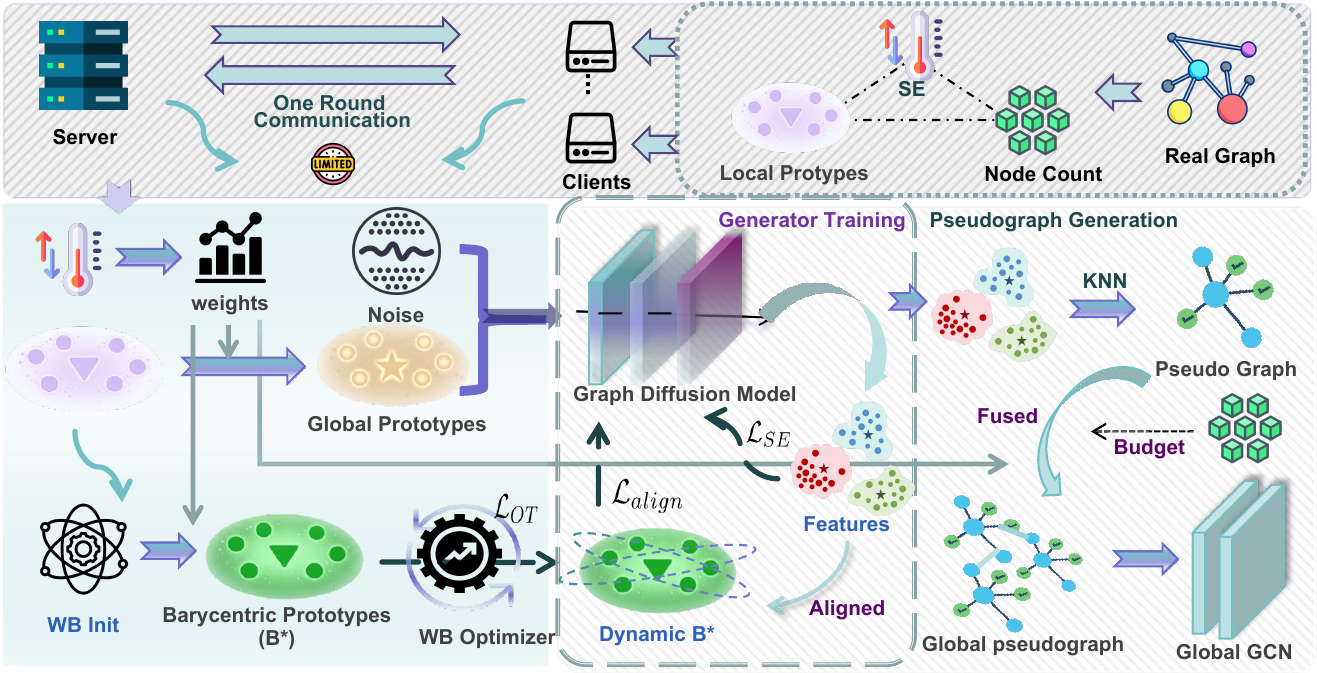}
    \caption{
    Overview of SPIRE.
    Clients upload structural entropy, class prototypes, and node counts in a single communication round.
    The server performs structural-entropy-based weighting, Wasserstein Barycenter initialization, pseudograph generation, and global GNN training.
    }
    \label{fig:architecture}
\end{figure*}

\subsection{Topology-Aware Influence Allocation}
\label{sec:quality}
In one-shot FGL, the server has no opportunity to iteratively
adjust client contributions after aggregation. The server must
therefore determine client influence from the limited statistics
available in a single communication round. This makes the choice
of a structural signal important: it should distinguish client
graphs beyond their volume without requiring access to their raw
topologies.

\textbf{First-Order Structural Entropy.}
Structural entropy provides a compact way to characterize the
organization of graph connectivity. While higher-order formulations
can capture hierarchical community structure, SPIRE uses the first-order form~\citep{DEHMER2008info} because our client-level weighting specifically requires a direct descriptor of degree-level connectivity organization.
Given a local graph $\mathcal{G}_k=(\mathcal{V}_k,\mathcal{E}_k)$, let
$d_v$ denote the degree of node $v$ and
$\mathrm{vol}(\mathcal{G}_k)=\sum_{v\in\mathcal{V}_k}d_v$ denote the
total degree volume. We define the first-order structural entropy as
the Shannon entropy of the normalized degree distribution:
\begin{equation}
H(\mathcal{G}_k)
=
-\sum_{v\in\mathcal{V}_k}
\frac{d_v}{\mathrm{vol}(\mathcal{G}_k)}
\log_2
\frac{d_v}{\mathrm{vol}(\mathcal{G}_k)}.
\label{eq:se_def}
\end{equation}
To normalize the entropy across clients with different sizes, we divide by the maximum entropy of a distribution over $N_k$ nodes:
\begin{equation}
S_k=
\frac{H(\mathcal{G}_k)}{\log_2 N_k},
\quad S_k\in[0,1],
\label{eq:se}
\end{equation}
where $N_k$ denotes the node count of client $k$ ($N_k>1$). A lower
$S_k$ indicates a more concentrated degree-mass distribution, whereas
a higher $S_k$ indicates a more dispersed one. Thus, $S_k$ provides a
first-order structural descriptor that complements client volume
$N_k$.

\textbf{Local Metadata and Client Weighting.}
Each client $k$ summarizes its local graph as a lightweight 
triplet $\{S_k, \mathcal{P}_k, N_k\}$, where 
$\mathcal{P}_k = \{\mathbf{p}_k^{(c)}\}_{c \in \mathcal{C}_k}$ 
denotes the set of class prototypes. Each prototype is computed directly on the raw feature space as the $\ell_2$-normalized class mean:
\begin{equation}
    \mathbf{p}_k^{(c)} = \frac{\mathbf{z}_k^{(c)}}
    {\|\mathbf{z}_k^{(c)}\|_2}, \quad 
    \mathbf{z}_k^{(c)} = \frac{1}{|\mathcal{V}_k^{(c)}|}
    \sum_{v \in \mathcal{V}_k^{(c)}} \mathbf{x}_v.
    \label{eq:proto}
\end{equation}
No raw features, adjacency matrices, or model parameters are 
transmitted, keeping the communication cost at $\mathcal{O}(C \times F + 1)$ per client. Upon receiving the statistics, the server maps each $S_k$ to an aggregation weight via a temperature-scaled softmax:
\begin{equation}
    \alpha_k = \frac{\exp(-S_k / \tau)}
    {\sum_{j=1}^K \exp(-S_j / \tau)},
    \label{eq:alpha_k}
\end{equation}
where $\tau > 0$ controls the weighting distribution. This inverse monotone mapping establishes a structural prior: it prioritizes clients with concentrated degree distributions, grounded in our empirical finding (Fig.~\ref{fig:motivation}) that such topologies retain clearer organizational signals under non-IID partitioning. Formally, this weighting is the unique solution to an entropy-regularized optimization problem.

\begin{theorem}[Entropy-Regularized Optimality]
\label{thm:optimality}
For $\tau>0$, the weight vector $\bm{\alpha}^* \in \Delta^{K-1}$ in Eq.~\eqref{eq:alpha_k} is the unique global minimizer of
\begin{equation}
    \bm{\alpha}^* = \arg\min_{\bm{\alpha} \in \Delta^{K-1}} \left[ \sum_{k=1}^K \alpha_k S_k + \tau \cdot \mathrm{KL}(\bm{\alpha} \| \mathbf{u}) \right], \nonumber
\end{equation}
where $u_k = 1/K$. Moreover, the weighted average structural entropy of the aggregated system is bounded by:
\begin{equation}
    0 \le \sum_{k=1}^K \alpha_k^* S_k - S_{\min} \le \tau \log K, \nonumber
\end{equation}
where $S_{\min}=\min_k S_k$.
\end{theorem}

\begin{proof}
Since $\mathrm{KL}(\bm{\alpha}\|\mathbf{u})$ is strictly convex on $\Delta^{K-1}$ and $\tau>0$, the objective is strictly convex and therefore has a unique global minimizer. Its Lagrangian is $\mathcal{L}(\bm{\alpha},\lambda) = \sum_k\alpha_kS_k + \tau\sum_k\alpha_k\log(K\alpha_k) + \lambda(\sum_k\alpha_k-1)$. For an interior minimizer, setting $\partial\mathcal{L}/\partial\alpha_k=0$ gives $S_k+\tau(\log(K\alpha_k)+1)+\lambda=0$, hence $\alpha_k \propto \exp(-S_k/\tau)$. Normalizing over $k$ recovers Eq.~\eqref{eq:alpha_k}. For the bound, let $J(\bm{\alpha})$ denote the objective and $\mathbf{e}_{\min}$ be any one-hot vector corresponding to a client attaining $S_{\min}$. Since $\bm{\alpha}^*$ is optimal, $J(\bm{\alpha}^*) \le J(\mathbf{e}_{\min}) = S_{\min} + \tau\,\mathrm{KL}(\mathbf{e}_{\min}\|\mathbf{u}) = S_{\min}+\tau\log K$. Because $\mathrm{KL}(\bm{\alpha}^*\|\mathbf{u})\ge0$ and $\sum_k\alpha_k^*S_k\ge S_{\min}$, we obtain $0 \le \sum_k\alpha_k^*S_k-S_{\min} \le \tau\log K$.
\end{proof}

Theorem~\ref{thm:optimality} shows that $\tau$ is not merely a numerical smoothing parameter; it controls the trade-off between favoring lower-entropy clients and retaining regularization toward uniform aggregation. 

To examine the sensitivity of the weighting to bounded perturbations in the transmitted structural statistic, suppose that the uploaded scalar $S_k$ suffers from estimation errors.

\begin{theorem}[Robustness to Entropy Perturbations]
\label{thm:robustness}
Suppose the transmitted structural entropy is perturbed such that $|\hat{S}_k - S_k| \le \delta$ for all $k$. Let $\bm{\alpha}$ and $\hat{\bm{\alpha}}$ be the weights computed from $\{S_k\}$ and $\{\hat{S}_k\}$, respectively. Then
\begin{equation}
    \exp\left(-\frac{2\delta}{\tau}\right) \le \frac{\hat{\alpha}_k}{\alpha_k} \le \exp\left(\frac{2\delta}{\tau}\right), \quad \|\hat{\bm{\alpha}} - \bm{\alpha}\|_1 \le 2\tanh\left(\frac{\delta}{\tau}\right). \nonumber
\end{equation}
\end{theorem}

\begin{proof}
Let $\Delta_k=-(\hat{S}_k-S_k)/\tau$, where $|\Delta_k|\le\delta/\tau$. Then $\hat{\alpha}_k = \alpha_k \exp(\Delta_k) / \sum_j\alpha_j \exp(\Delta_j)$. For any $i,j$, $\frac{\hat{\alpha}_i/\alpha_i}{\hat{\alpha}_j/\alpha_j} = \exp(\Delta_i-\Delta_j)$, and hence $\exp(-2\delta/\tau) \le \frac{\hat{\alpha}_i/\alpha_i}{\hat{\alpha}_j/\alpha_j} \le \exp(2\delta/\tau)$. Since $\sum_k\alpha_k(\hat{\alpha}_k/\alpha_k)=1$, this implies $\exp(-2\delta/\tau) \le \hat{\alpha}_k/\alpha_k \le \exp(2\delta/\tau)$. Let $R=\exp(2\delta/\tau)$. Then $1/R\le \hat{\alpha}_k/\alpha_k\le R$, and for any $r\in[1/R,R]$, we have the inequality $|r-1| \le \frac{R-1}{R+1}(r+1)$. Therefore, $\|\hat{\bm{\alpha}}-\bm{\alpha}\|_1 = \sum_k\alpha_k \left|\frac{\hat{\alpha}_k}{\alpha_k}-1\right| \le \frac{R-1}{R+1} \sum_k\alpha_k\left(\frac{\hat{\alpha}_k}{\alpha_k}+1\right) = 2\frac{R-1}{R+1} = 2\tanh(\delta/\tau)$.
\end{proof}

Theorem~\ref{thm:robustness} provides a bounded sensitivity guarantee for entropy perturbations. Notably, a larger $\tau$ dampens the error transmission. Thus, $\tau$ simultaneously controls structural selectivity and error sensitivity. Rather than using these robust weights $\alpha_k$ solely for an initial scalar aggregation, SPIRE treats them as a shared influence signal that dictates how client information contributes to the subsequent global reconstruction. This persistent propagation ensures that the structural preference is not lost, forming the central integration mechanism of our one-shot framework.

\subsection{Structure-Guided Pseudograph Synthesis}
\label{sec:generation}

Having obtained the entropy-derived weights $\alpha_k$, the
server performs pseudograph synthesis. The same weights are propagated
through semantic initialization, geometric alignment, feature
generation, and pseudograph assembly.

\textbf{Entropy-Weighted Global Initialization.}
The server first constructs two complementary anchors that 
jointly condition the diffusion process: a semantic 
condition $\bm{\mathcal{P}}_{con} = \sum_{k=1}^K \alpha_k \bm{\mathcal{P}}_k$ and a geometric target 
$\mathcal{B}^*$.

Since $\bm{\mathcal{P}}_{con}$ captures only first-order statistics, we optimize a geometric target 
$\mathcal{B}^*$ via the Wasserstein 
Barycenter~\citep{agueh2011barycenters}, using the exact same weights $\alpha_k$ to maintain structural consistency:
\begin{equation}
    \mathcal{B}^* = \mathop{\arg\min}_{\mathcal{B}} 
    \sum_{k=1}^K \alpha_k \cdot 
    \mathcal{L}_{OT}(\mathcal{B}, \bm{\mathcal{P}}_k),
    \label{eq:b*}
\end{equation}
where $\mathcal{L}_{OT}$ is the Sinkhorn optimal transport cost. To characterize how entropy estimation errors propagate to the barycentric objective, we establish the following lemma and corollary.

\begin{lemma}[Barycentric Objective Stability]
\label{lem:stability}
Assume that the admissible barycenters lie in a bounded feasible set $\Omega\subset\mathbb{R}^F$ and that $0\le\mathcal{L}_{OT}(\mathcal{B},\mathcal{P}_k)\le M$ for all admissible $\mathcal{B}$ and clients $k$. Define $V(\bm{\alpha}) = \inf_{\mathcal{B}\in\Omega} \sum_k\alpha_k \mathcal{L}_{OT}(\mathcal{B},\mathcal{P}_k)$. Then $|V(\hat{\bm{\alpha}})-V(\bm{\alpha})| \le M\|\hat{\bm{\alpha}}-\bm{\alpha}\|_1$.
\end{lemma}

\begin{proof}
Let $F(\mathcal{B}; \bm{\alpha}) = \sum_k \alpha_k \mathcal{L}_{OT}(\mathcal{B}, \bm{\mathcal{P}}_k)$. For any fixed admissible $\mathcal{B}$, $|F(\mathcal{B}; \hat{\bm{\alpha}}) - F(\mathcal{B}; \bm{\alpha})| = \left| \sum_k (\hat{\alpha}_k - \alpha_k) \mathcal{L}_{OT}(\mathcal{B}, \bm{\mathcal{P}}_k) \right| \le M \|\hat{\bm{\alpha}} - \bm{\alpha}\|_1$. Since the bound holds uniformly over $\mathcal{B}$, $V(\hat{\bm{\alpha}}) \le V(\bm{\alpha}) + M\|\hat{\bm{\alpha}}-\bm{\alpha}\|_1$. Swapping $\bm{\alpha}$ and $\hat{\bm{\alpha}}$ gives the reverse inequality, yielding the stated absolute bound.
\end{proof}

\begin{corollary}[Barycentric Robustness to Entropy Perturbations]
\label{cor:robustness}
Combining Theorem~\ref{thm:robustness} and Lemma~\ref{lem:stability}, the initialization objective deviation under bounded entropy perturbation $\delta$ is bounded by $|V(\hat{\bm{\alpha}}) - V(\bm{\alpha})| \le 2M \tanh(\delta/\tau)$.
\end{corollary}

Corollary~\ref{cor:robustness} bounds the propagation of entropy perturbations through client weighting to the barycentric objective. To prevent early-stage geometric drift in practice, $\mathcal{B}^*$ is warm-started via static WB alignment and held fixed for the first quarter of training; subsequently, $\mathcal{B}^*$ and the generator $\theta$ are jointly optimized.

\begin{algorithm}[h]
\caption{The Main Steps of SPIRE}
\label{alg:spire}
\begin{algorithmic}[1]
\renewcommand{\algorithmicrequire}{\textbf{Input:}}
\renewcommand{\algorithmicensure}{\textbf{Output:}}

\Require Local datasets $\{\mathcal{G}_k\}_{k=1}^K$; Server epochs $T_G$, Gen steps $K_{gen}$; Hyperparameters $\tau, \lambda_{align}, \lambda_{se}$.
\Ensure Global GNN parameters $\mathbf{W}_{global}$.

% --- 使用 vspace 撑开区块间距，创造视觉呼吸感 ---
\vspace{0.12cm}
\Statex \textbf{[Client Side]}
\For{client $k = 1, \dots, K$ \textbf{in parallel}}
    \State Compute $S_k$ (Eq.~\eqref{eq:se}) and $\bm{\mathcal{P}}_k$ (Eq.~\eqref{eq:proto})
    \State Upload $\{S_k, \bm{\mathcal{P}}_k, N_k\}$ to server.
\EndFor

\vspace{0.12cm}
\Statex \textbf{[Server Side]}
\State Compute weights $\alpha_k$ (Eq.~\eqref{eq:alpha_k}), and construct condition $\bm{\mathcal{P}}_{con}$ 
\State Init $\mathcal{B}^*$, generator $\theta$, and $\mathbf{W}_{global}$

\vspace{0.08cm}
\For{$s = 1$ \textbf{to} $T_G$}
    \vspace{0.06cm}
    \State \textit{\# Phase 1: Target Refinement \& Generator Optimization}
    \State Update weight $\gamma(s)$ (Eq.~\eqref{eq:lambda_schedule})
    \State Refine geometric target $\mathcal{B}^*$ (Eq.~\eqref{eq:b*})
    \For{$step = 1$ \textbf{to} $K_{gen}$} 
        \State Sample $\mathbf{Z} \sim \mathcal{N}(\mathbf{0}, \mathbf{I})$
        \State Synthesize $\hat{\mathbf{X}}$ and $\hat{\mathbf{A}}_{soft}$ (Eq.~\eqref{eq:a_soft}) via Generator $\theta$
        \State Update $\theta$ by minimizing Joint Loss $\mathcal{L}$ (Eq.~\eqref{eq:total_loss})
    \EndFor
    
    \vspace{0.1cm}
    \State \textit{\# Phase 2: Pseudograph Synthesis \& Global Training}
    \State Initialize pseudograph memory $\mathcal{G}_{pseudo} = \emptyset$.
    \For{$k = 1$ \textbf{to} $K$}
        \State Determine generation budget $N_{gen}^{(k)}$ (Eq.~\eqref{eq:budget})
        \State Generate jittered features $\hat{\mathbf{X}}_{final}^{(k)}$ 
        \State Build KNN topology $\mathbf{A}^{(k)}$ to expand $\mathcal{G}_{pseudo}$
    \EndFor
    \State Train global GNN $\mathbf{W}_{global}$ via Cross-Entropy loss on $\mathcal{G}_{pseudo}$
    \vspace{0.06cm}
\EndFor

\vspace{0.12cm}
\State \Return $\mathbf{W}_{global}$
\end{algorithmic}
\end{algorithm}

\begin{table*}[t]
\centering
\caption{Performance comparison on node classification tasks. The abbreviations in the first column denote: \textbf{BL} (Baseline), and \textbf{OS} (One-shot Federated Learning). The best results are highlighted in \textbf{bold}, and the second-best results are \underline{underlined}. The small numbers indicate the performance gap compared to FedAvg (\textcolor{mygreen}{$\uparrow$} improvement, \textcolor{myred}{$\downarrow$} degradation).}
\label{tab:main_results}

\resizebox{\textwidth}{!}{%
\begin{tabular}{c|l||ccccccc}
\toprule
\textbf{Type} & \textbf{Methods} & \textbf{Cora} & \textbf{CiteSeer} & \textbf{PubMed} & \textbf{Amz-Comp} & \textbf{WikiCS} & \textbf{Coauthor-CS} & \textbf{ogbn-arxiv} \\ 
\midrule
% 基准 BL
BL & FedAvg \myconf{[AISTATS17]} & 34.74 & 36.08 & 61.92 & 37.70 & 19.47 & 25.33 & 14.58 \\ 
\midrule
% Traditional FL
\multirow{3}{*}{FL} 
& FedProx \myconf{[MLSys20]} & 28.24\down{6.50} & 30.48\down{5.60} & 54.62\down{7.30} & 24.74\down{12.96} & 4.22\down{15.25} & 23.33\down{2.00} & 13.37\down{1.21} \\
& FedOPT \myconf{[ICLR21]} & 29.68\down{5.06} & 21.82\down{14.26} & 43.86\down{18.06} & 12.42\down{25.28} & 10.57\down{8.90} & 7.76\down{17.57} & 2.14\down{12.44} \\
& MOON \myconf{[CVPR21]} & 27.98\down{6.76} & 22.12\down{13.96} & 45.50\down{16.42} & 23.71\down{13.99} & 4.13\down{15.34} & 11.44\down{13.89} & 11.57\down{3.01} \\ 
\midrule
% Traditional FGL
\multirow{4}{*}{FGL} 
& FedProto \myconf{[AAAI22]} & 30.50\down{4.24} & 24.56\down{11.52} & 52.86\down{9.06} & 37.66\down{0.04} & 6.67\down{12.80} & 26.51\up{1.18} & 5.20\down{9.38} \\
& FedPub \myconf{[ICML23]} & 21.16\down{13.58} & 19.16\down{16.92} & 41.16\down{20.76} & 37.70\down{0.00} & 17.49\down{1.98} & 8.51\down{16.82} & 10.16\down{4.42} \\
& FedGTA \myconf{[VLDB24]} & 27.06\down{7.68} & 20.92\down{15.16} & 42.98\down{18.94} & \second{38.22}\up{0.52} & 6.63\down{12.84} & 7.65\down{17.68} & 1.77\down{12.81} \\ 
& FedTAD \myconf{[IJCAI24]} & 22.92\down{11.82} & 22.48\down{13.60} & 46.96\down{14.96} & 31.71\down{5.99} & 16.25\down{3.22} & 4.62\down{20.71} & 1.27\down{13.31} \\
\midrule
% One-shot FL
\multirow{6}{*}{OSFL} 
& DENSE \myconf{[NeurIPS22]} & 25.10\down{9.64} & 18.36\down{17.72} & 41.70\down{20.22} & 14.94\down{22.76} & \second{17.83}\down{1.64} & 16.22\down{9.11} & 5.86\down{8.72} \\
& FedCVAE \myconf{[ICLR23]} & 14.78\down{19.96} & 17.17\down{18.91} & 33.68\down{28.24} & 10.86\down{26.84} & 17.36\down{2.11} & 11.97\down{13.36} & 11.52\down{3.06} \\
& FedSD2C \myconf{[NeurIPS24]} & 21.68\down{13.06} & 18.66\down{17.42} & 32.90\down{29.02} & 12.94\down{24.76} & 12.18\down{7.29} & 10.03\down{15.30} & 5.87\down{8.71} \\ 
& GHOST \myconf{[ICML25]} & 38.19\up{3.45} & 38.95\up{2.87} & 59.03\down{2.89} & 36.59\down{1.11} & 11.95\down{7.52} & 28.31\up{2.98} & \second{15.53}\up{0.95} \\
& OASIS \myconf{[NeurIPS25]} & \second{41.66}\up{6.92} & \second{41.29}\up{5.21} & \second{62.48}\up{0.56} & 36.83\down{0.87} & 12.46\down{7.01} & \second{28.70}\up{3.37} & OOM \\
\cmidrule{2-9}
% Ours 行高亮背景色
& \cellcolor{lightyellow}\textbf{SPIRE (Ours)} & \cellcolor{lightyellow}\best{64.06}\up{29.32} & \cellcolor{lightyellow}\best{55.32}\up{19.24} & \cellcolor{lightyellow}\best{73.40}\up{11.48} & \cellcolor{lightyellow}\best{38.32}\up{0.62} & \cellcolor{lightyellow}\best{21.06}\up{1.59} & \cellcolor{lightyellow}\best{29.65}\up{4.32} & \cellcolor{lightyellow}\best{21.04}\up{6.46} \\ 
\bottomrule
\end{tabular}%
}
\end{table*}

\textbf{Conditional Feature Generation.}
With $\bm{\mathcal{P}}_{con}$ and $\mathcal{B}^*$ 
established as dual anchors, the server employs a 
conditional diffusion model~\citep{ho2020ddpm} to 
generate the node feature matrix $\hat{\mathbf{X}}$. The 
reverse process iteratively denoises 
$\mathbf{X}_T \sim \mathcal{N}(\mathbf{0}, \mathbf{I})$ 
conditioned on $\bm{\mathcal{P}}_{con}$:
\begin{equation}
    \mathbf{X}_{t-1} = \frac{1}{\sqrt{\rho_t}} \left( 
    \mathbf{X}_t - \frac{1-\rho_t}{\sqrt{1-\bar{\rho}_t}} 
    \boldsymbol{\epsilon}_\theta(\mathbf{X}_t, t, 
    \bm{\mathcal{P}}_{con}) \right) + \sigma_t \mathbf{Z},
    \label{eq:gfm}
\end{equation}
where $\rho_t$ is the signal retention schedule and $\boldsymbol{\epsilon}_\theta$ is the learnable noise prediction network. 

This feature-first paradigm separates feature synthesis
from direct edge reconstruction. By generating structure-aware features conditioned on entropy-weighted prototypes and constraining their distribution, the resulting topology instantiated through KNN is determined by the generated feature geometry rather than inherited from noisy local topologies. We optimize the generator via a dual objective:
\begin{equation}
    \mathcal{L} = \lambda_{align} \mathcal{L}_{align} + \gamma(s) \mathcal{L}_{SE}.
    \label{eq:total_loss}
\end{equation}

The first term $\mathcal{L}_{align}$ closes the loop 
established by Eq.~\eqref{eq:b*}: it minimizes the 
discrepancy between synthesized class centroids $\bar{\mathbf{x}}_c$ and the 
corresponding geometric target $\mathbf{b}^*_c$:
\begin{equation}
    \mathcal{L}_{align} = \frac{1}{|\mathcal{C}_B|} 
    \sum_{c \in \mathcal{C}_B} \left\| \bar{\mathbf{x}}_c 
    - \mathbf{b}^*_c \right\|_2^2.
    \label{eq:l_align}
\end{equation}

The second term $\mathcal{L}_{SE}$ introduces a structural
regularization consistent with the entropy-based weighting scheme. It
encourages the generated topology to follow the same degree-level
structural preference utilized for client weighting. We construct a differentiable soft adjacency matrix:
\begin{equation}
    \hat{\mathbf{A}}_{soft} = \text{ReLU}\left( 
    \tilde{\mathbf{X}} \tilde{\mathbf{X}}^\top \right),
    \label{eq:a_soft}
\end{equation}
where $\tilde{\mathbf{X}}$ is the $\ell_2$-normalized $\hat{\mathbf{X}}$. $\mathcal{L}_{SE}$ is then 
computed via Eq.~\eqref{eq:se_def} on $\hat{\mathbf{A}}_{soft}$, with the node degree explicitly defined as the row sum of the continuous edge weights: $d_v = \sum_u \hat{\mathbf{A}}_{vu}$.

To avoid structural instability before generated 
features converge toward class-discriminative 
regions, we introduce a cosine warm-up schedule 
that gradually activates the constraint:
\begin{equation}
    \gamma(s) = 
    \begin{cases} 
    \frac{\lambda_{se}}{2} \left( 1 - \cos \left( 
    \frac{s \pi}{T_{warm}} \right) \right), 
    & s < T_{warm} \\
    \lambda_{se}, & \text{otherwise}
    \end{cases}
    \label{eq:lambda_schedule}
\end{equation}
where $s$ is the current training step.

\textbf{Structure-Aware Pseudograph Assembly.}
Synthesized features are assembled into the global pseudograph $\mathcal{G}_{pseudo}$. Within a base interval
$[N_{min}^{(k)}, N_{max}^{(k)}]$, $\alpha_k$ strictly determines the generated
sample count:
\begin{equation}
\label{eq:budget}
    N_{gen}^{(k)} = \left\lfloor N_{min}^{(k)} +
    \varsigma \cdot \alpha_k \cdot (N_{max}^{(k)} -
    N_{min}^{(k)}) \right\rceil.
\end{equation}
We inject controlled Gaussian noise
$\boldsymbol{\xi} \sim \mathcal{N}(\mathbf{0}, \sigma^2\mathbf{I})$
into the synthesized features to reduce deterministic
overfitting: $\hat{\mathbf{X}}_{final} = \hat{\mathbf{X}} + \boldsymbol{\xi}$.
Topology is then reconstructed for each
$\hat{\mathbf{X}}_{final}^{(k)}$ via KNN~\citep{cover1967nearest}. Client-wise subgraphs are merged
via disjoint union, with edge weights permanently scaled by $\alpha_k$.

\textbf{Complexity Analysis.}
On the client side, computing structural entropy and class 
prototypes requires a single forward pass over 
the local graph, costing 
$\mathcal{O}(|\mathcal{V}_k|F + |\mathcal{E}_k|)$ 
with zero local backpropagation, where 
$|\mathcal{V}_k|$ and $|\mathcal{E}_k|$ denote 
the node and edge counts of client $k$, and $F$ 
is the feature dimension. Communication overhead 
is strictly $\mathcal{O}(CF + 1)$ per client, 
where $C$ is the number of classes. On the server 
side, Wasserstein Barycenter optimization via the 
Sinkhorn algorithm costs $\mathcal{O}(IKC^2)$, 
where $I$ is the number of Sinkhorn iterations 
and $K$ is the number of clients; diffusion-based 
feature synthesis over $N$ total nodes with $T$ 
denoising steps and hidden dimension $H$ costs 
$\mathcal{O}(TNH^2)$; and KNN topology 
reconstruction costs $\mathcal{O}(N^2F)$. The 
overall complexity is 
$\mathcal{O}(IKC^2 + TNH^2 + N^2F)$, where the 
dominant server-side terms formally confirm the 
intended asymmetric design.

Algorithm~\ref{alg:spire} summarizes the complete 
procedure of SPIRE, reflecting the two-sided design 
established in Sections~\ref{sec:quality} 
and~\ref{sec:generation}.

\begin{table*}[t]
\centering
\caption{Efficiency and Communication Overhead Comparison. Each cell presents \textbf{Comm. Time / Comm. Cost / Total Training Time}. $T_{train}$ for SPIRE includes server-side diffusion training, while others denote cumulative training time.}
\label{tab:efficiency}
\resizebox{\textwidth}{!}{%
\begin{tabular}{l||c|c|c|c}
\toprule
\best{Methods} & \best{Cora} & \best{CiteSeer} & \best{PubMed} & \best{ogbn-arxiv} \\ 
\midrule
FedAvg \myconf{(200 rounds)} & 65.71 s / 2814 MB / 66 s & 67.33 s / 7246 MB / 68 s & 66.45 s / 984 MB / 67 s & 79.17 s / 330 MB / 80 s \\ 
GHOST \myconf{[ICML25]}      & 40.36 s / 16.0 MB / 71 s & 60.00 s / 38.0 MB / 102 s & 5471 s / 7.0 MB / 5716 s & 18106 s / 3.0 MB / 20028 s \\ 
OASIS \myconf{[NeurIPS25]}   & 172 s / 44.0 MB / 179 s  & 179 s / 111 MB / 189 s   & 6971 s / 17.0 MB / 7004 s & \conf{OOM} \\ 
\midrule
\rowcolor{lightyellow} 
\best{SPIRE (Ours)} & \best{<0.01 s / 0.38 MB} / 1265 s & \best{<0.01 s / 0.85 MB} / 1227 s & \best{<0.01 s / 0.06 MB} / 986 s & \best{0.11 s / 0.20 MB} / 653 s \\ 
\bottomrule
\end{tabular}%
}
\end{table*}

\section{Experiments}
\subsection{Experimental Setup}

\textbf{Datasets.}
We evaluate SPIRE on seven node classification 
benchmarks covering diverse domains and scales: 
Cora~\citep{mccallum2000automating}, 
CiteSeer~\citep{giles1998citeseer}, 
PubMed~\citep{sen2008collective}, 
Amazon-Computers~\citep{shchur2018pitfalls}, 
WikiCS~\citep{mernyei2020wikics}, 
Coauthor-CS~\citep{shchur2018pitfalls}, and 
ogbn-arxiv~\citep{hu2020open}. For Cora, 
CiteSeer, and PubMed, we adopt the standard 
Planetoid split (20 nodes per class for training, 
500 validation, 1,000 test). For ogbn-arxiv and 
WikiCS, we follow their official benchmark splits. 
For Computers and Coauthor-CS, we apply random 
partitions at 20\%/40\%/40\%.

\textbf{Baselines.}
We compare against four conventional FL methods: 
FedAvg [AISTATS17] \citep{mcmahan2017communication}, 
FedProx {[MLSys20] \citep{li2020fedprox}, 
FedOPT [ICLR21] \citep{reddi2021adaptive}, 
and MOON [CVPR21] \citep{li2021moon}; 
four multi-round FGL methods: 
FedProto [AAAI22] \citep{tan2022fedproto}, 
FedPub [ICML23] \citep{baek2023personalized}, 
FedGTA [VLDB23] \citep{li2023fedgta}, 
and FedTAD [IJCAI24] \citep{zhu2024fedtad}; 
and five one-shot FL/FGL methods: 
DENSE [NeurIPS22] \citep{zhang2022dense}, 
FedCVAE [ICLR23] \citep{heinbaugh2023data}, 
FedSD2C [NeurIPS24] \citep{zhang2024one}, 
GHOST [ICML25] \citep{qian2025ghost}, 
and OASIS [NeurIPS25] \citep{wanoasis}.

\textbf{Implementation Details.}
All experiments are conducted on an NVIDIA GeForce RTX 4090 GPU with PyTorch 2.4.1 and PyTorch Geometric 2.6.1. For strict and fair comparison, all federated baselines and SPIRE share the identical backbone classifier: a two-layer GCN with hidden dimension 256, dropout ratio 0.3, and ReLU activation, evaluated under identical transductive splits. For SPIRE, the MLP-based conditional diffusion generator uses the same hidden dimension (256), $T=10$ denoising steps, and AdaLN-based condition injection. Data is partitioned among $K=10$ clients via a Dirichlet distribution $\text{Dir}(\beta=0.05)$ to simulate extreme non-IID conditions. All modules are optimized using Adam (weight decay $5\times10^{-4}$). The global model training steps per round is fixed to $T_G=30$ with a server learning rate of $server\_lr=0.005$, while the learning rate for client local models (for multi-round baselines) is 0.01, for the diffusion generator is 0.001, and for the Wasserstein Barycenter learner is $5\times10^{-4}$. The server trains for up to 300 epochs (with reduced epochs and increased jittering on ogbn-arxiv). Key hyperparameters include $\tau=1.0$, $\lambda_{se} \in \{10^{-5},10^{-3},10^{-1}\}$, $\lambda_{align} \in \{0.1, 0.5, 1.0, 5.0, 10.0\}$, and $gen\_knn \in \{6,8,10,12\}$, all selected via grid search. The Wasserstein Barycenter is optimized via the Sinkhorn algorithm with 3 iterations. All reported results are averaged over five independent runs with different random seeds. Code is available at \url{https://github.com/Yodeesy/SPIRE}.

\subsection{Performance Comparison}
\textbf{General Classification Performance.} 
Table~\ref{tab:main_results} reports node classification accuracy under $K=10$, $\beta=0.05$.

SPIRE achieves the strongest performance across all seven datasets, consistently outperforming the one-shot baselines. On Cora, it surpasses OASIS by 22.40\% and GHOST by 25.87\%; on ogbn-arxiv, it exceeds GHOST by 5.51\%, while OASIS runs out of memory. The gains reflect the combined effect of topology-aware client weighting and structure-guided pseudograph generation, which are not explicitly modeled by existing one-shot baselines.

Multi-round FGL methods degrade substantially when restricted to a single communication round. FedGTA scores 27.06\% on Cora, below FedAvg's 34.74\%, and FedTAD drops to 22.92\%. These results highlight the difficulty of transferring methods that rely on iterative cross-client alignment to the strict one-shot setting.

Among one-shot baselines, two limitations are evident. FedCVAE and FedSD2C do not explicitly model graph topology during knowledge transfer. Conversely, OASIS incorporates topology through a synthetic graph construction process but runs out of memory on ogbn-arxiv, exposing a scalability limitation when topology-aware synthesis is applied to a large graph. SPIRE combines topology-aware client weighting with server-side pseudograph synthesis, avoiding these two limitations in the evaluated settings.

\textbf{Communication Efficiency and Scalability.}
SPIRE explicitly decouples client representation from global topology generation. By eliminating iterative backpropagation and requiring only $\mathcal{O}(CF+1)$, client-side communication remains under 0.01s on Cora, CiteSeer, and PubMed (e.g., 0.38 MB on Cora). Even on the large ogbn-arxiv graph, SPIRE requires only 0.11s and 0.20 MB per client, maintaining the one-shot constraint.

Total training time exhibits a pragmatic computation--utility trade-off. On Cora, SPIRE requires 1265s, primarily driven by server-side synthesis. To alleviate local data scarcity, the dynamically allocated generation budget $N_{gen}^{(k)}$ (Eq.~\ref{eq:budget}) provides additional synthesized samples. The resulting additional server-side computation accompanies the substantially higher accuracy observed on Cora. Although pairwise feature-similarity computation introduces an $\mathcal{O}(N^2F)$ worst-case cost, deliberately shifting this computationally intensive reconstruction to the server holds profound practical significance. In real-world web applications, this asymmetric design is particularly suitable for cross-silo deployments, where server-side compute is available but repeated communication is costly.

On ogbn-arxiv, methods relying on direct topology-aware construction face severe bottlenecks: OASIS runs out of memory, and GHOST requires over 18,000s. In contrast, SPIRE finishes in just 653s. Furthermore, Although FedAvg has low measured runtime in the simulation (66s on Cora), its 200 rounds require substantially more communication exchanges in bandwidth-constrained deployments. Overall, SPIRE achieves a favorable balance between communication cost, server compute, and graph-mining utility.

\subsection{Robustness Analysis}

\begin{figure}[t]
	\centering
	% === 左图：Cora ===
	\begin{minipage}[b]{0.48\linewidth}
		\centering
		\includegraphics[width=\linewidth]{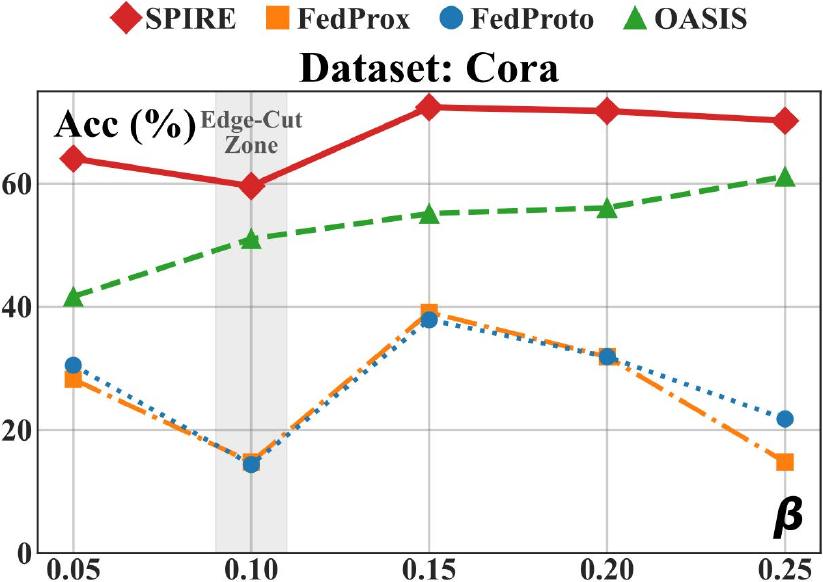}
		\centerline{\small (a) Cora}
	\end{minipage}
	\hfill % 左右推开
	% === 右图：CiteSeer ===
	\begin{minipage}[b]{0.48\linewidth}
		\centering
		\includegraphics[width=\linewidth]{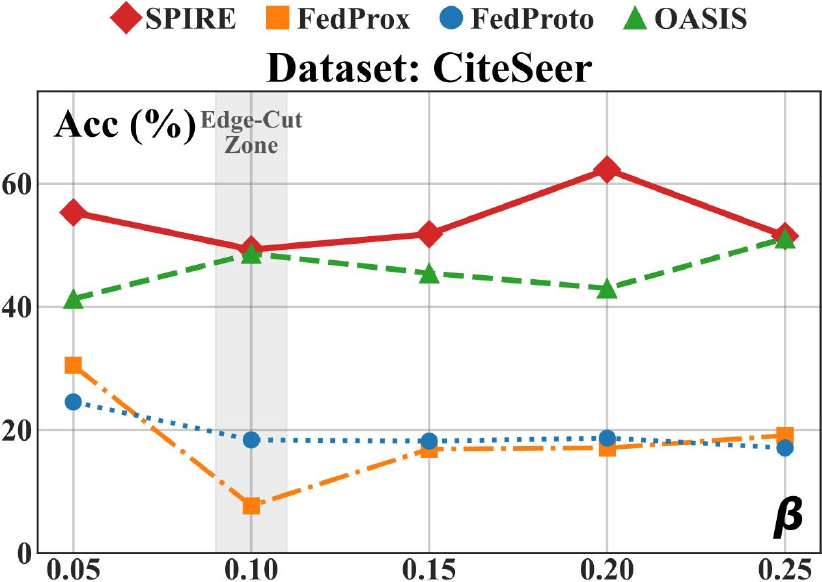}
		\centerline{\small (b) CiteSeer}
	\end{minipage}
    \caption{Impact of data heterogeneity ($\beta \in [0.05, 0.25]$) on datasets Cora and Citeseer.}
	\label{fig:heterogeneity}
    \vspace{-0.2cm}
\end{figure}

\textbf{Impact of Data Heterogeneity.}

Figure~\ref{fig:heterogeneity} tracks accuracy 
across $\beta \in [0.05, 0.25]$. As heterogeneity 
intensifies toward $\beta=0.05$, baselines exhibit 
substantial accuracy drops while SPIRE remains 
stable and widens its lead. This behavior is consistent with the topology-aware weighting in SPIRE: the aggregation weights depend on each client's degree-mass distribution rather than solely on its data volume, allowing structurally different clients to contribute unequally under severe label skew. In 
contrast, conventional aggregation strategies tend 
to propagate these skewed local biases uniformly, 
leading to degraded global representations under 
extreme non-IID settings.

\textbf{Resilience to Graph Perturbations.}
\begin{table}[!t]
\centering
\caption{Robustness analysis under edge/feature perturbation ($\rho$/$\eta$=$0.5$) on representative datasets of varying scales. The small numbers indicate the performance deviation compared to the clean setting (\textcolor{mygreen}{$\uparrow$} improvement, \textcolor{myred}{$\downarrow$} degradation).}
\label{tab:perturbation}
\vspace{2pt}
\resizebox{\linewidth}{!}{%
\begin{tabular}{l||cccc}
\toprule
\textbf{Methods} & \textbf{Cora} & \textbf{CiteSeer} & \textbf{PubMed} & \textbf{ogbn-arxiv} \\ 
\midrule
\multicolumn{5}{c}{\textbf{(a) Edge Perturbation ($\rho=0.5$)}} \\
\midrule
FedAvg & 31.90\down{2.84} & 18.10\down{17.98} & 42.40\down{19.52} & 22.28\up{7.70} \\
DENSE & 16.40\down{8.70} & 18.20\down{0.16} & 41.60\down{0.10} & 3.48\down{2.38} \\
FedCVAE & 15.10\up{0.32} & 17.20\up{0.03} & 41.00\up{7.32} & 21.56\up{10.04} \\
FedSD2C & 14.90\down{6.78} & 16.90\down{1.76} & 41.20\up{8.30} & 7.87\up{2.00} \\
GHOST & 32.75\down{5.44} & 40.60\up{1.65} & 48.94\down{10.09} & 17.54\up{2.01} \\
OASIS & 47.57\up{5.91} & 46.14\up{4.85} & 50.21\down{12.27} & OOM \\
\midrule
\rowcolor{lightyellow} 
\textbf{SPIRE} & \textbf{62.10}\down{1.96} & \textbf{58.20}\up{2.88} & \textbf{70.60}\down{2.80} & \textbf{26.65}\up{5.61} \\ 
\midrule
\multicolumn{5}{c}{\textbf{(b) Feature Perturbation ($\eta=0.5$)}} \\
\midrule
FedAvg & 31.90\down{2.84} & 18.20\down{17.88} & 48.30\down{13.62} & 10.22\down{4.36} \\
DENSE & 14.90\down{10.20} & 18.20\down{0.16} & 40.70\down{1.00} & 3.78\down{2.08} \\
FedCVAE & 31.90\up{17.12} & 13.80\down{3.37} & 40.80\up{7.12} & 12.55\up{1.03} \\
FedSD2C & 31.85\up{10.17} & 16.84\down{1.82} & 40.70\up{7.80} & 18.87\up{13.00} \\
GHOST & 33.46\down{4.73} & 36.85\down{2.10} & 47.11\down{11.92} & 12.79\down{2.74} \\
OASIS & 41.34\down{0.32} & 39.10\down{2.19} & 49.25\down{13.23} & OOM \\
\midrule
\rowcolor{lightyellow} 
\textbf{SPIRE} & \textbf{64.10}\up{0.04} & \textbf{51.05}\down{4.27} & \textbf{68.80}\down{4.60} & \textbf{21.87}\up{0.83} \\ 
\bottomrule
\end{tabular}%
}
\end{table}

As shown in Tab.~\ref{tab:perturbation}, we report the complete accuracy under edge perturbation ($\rho=0.5$) and feature perturbation ($\eta=0.5$) across several datasets. 
Most baselines suffer sharp declines: GHOST drops 
5.44\% on Cora under edge perturbation, and FedAvg 
loses 17.98\% on CiteSeer. SPIRE incurs only a 
1.96\% drop on Cora under edge perturbation, and 
on Cora under feature perturbation its accuracy is 
essentially unchanged at 64.10\%. This stability 
follows from the generation process itself: node features are synthesized from entropy-weighted prototypes rather than copied from potentially corrupted local graphs, and the $\mathcal{L}_{SE}$ term encourages the generated graph to follow the same degree-level structural preference used for client weighting, rather than directly reproducing the perturbed local topology.

\textbf{Scalability to Large Client Numbers.}
Table~\ref{tab:scalability} reports accuracy under 
intensified data fragmentation ($K \in \{5, 10, 
20, 50, 100\}$, $\beta=0.05$). As $K$ increases, 
each client receives a smaller and more skewed data 
shard, posing a fundamental challenge for 
pseudograph synthesis. OASIS runs out of memory on 
ogbn-arxiv across all $K$ settings, while GHOST 
and FedAvg exhibit consistent accuracy degradation 
beyond $K=20$. SPIRE maintains strong performance 
throughout: on Cora it achieves 60.35\% and 
58.58\% at $K=50$ and $K=100$ respectively, 
substantially above OASIS's 40.59\%/37.59\% and 
GHOST's 37.10\%/35.33\%. This stability is consistent with the entropy-based weighting mechanism: as the number of clients increases, SPIRE continues to differentiate client influence according to their degree-mass distributions rather than relying solely on increasingly fragmented client volumes. The resulting weighting allows heterogeneous clients to contribute unequally to pseudograph synthesis even when each client holds only a small data shard. These results indicate that SPIRE remains effective as client-side data becomes increasingly fragmented.

\begin{table}[t]
\centering
\caption{Scalability under severe data fragmentation ($\beta=0.05$) with varying numbers of clients $K$. Best results are bold.}
\label{tab:scalability}
\setlength{\tabcolsep}{3.5pt}
\small
\begin{tabular}{llccccc}
\toprule
\textbf{Dataset} & \textbf{Method} &
\textbf{$K=5$} & \textbf{$10$} & \textbf{$20$} & \textbf{$50$} & \textbf{$100$} \\
\midrule

\multirow{4}{*}{Cora}
& FedAvg
& 32.40 & 34.74 & 31.90 & 32.31 & 31.95 \\
& GHOST
& 41.54 & 38.19 & 34.15 & 37.10 & 35.33 \\
& OASIS
& 51.56 & 41.66 & 35.42 & 40.59 & 37.59 \\
& \textbf{SPIRE}
& \textbf{55.70} & \textbf{64.06} & \textbf{57.70} & \textbf{60.35} & \textbf{58.58} \\

\midrule

\multirow{4}{*}{PubMed}
& FedAvg
& 60.80 & 61.92 & 63.40 & 61.80 & 59.35 \\
& GHOST
& 60.52 & 59.03 & 56.37 & 57.35 & 58.37 \\
& OASIS
& 64.15 & 62.48 & 61.78 & 60.30 & 59.76 \\
& \textbf{SPIRE}
& \textbf{68.20} & \textbf{73.40} & \textbf{71.30} & \textbf{71.20} & \textbf{69.39} \\

\midrule

\multirow{4}{*}{CiteSeer}
& FedAvg
& 26.10 & 36.08 & 18.10 & 36.10 & 34.90 \\
& GHOST
& 36.03 & 38.95 & 43.51 & 37.75 & 36.71 \\
& OASIS
& 41.27 & 41.29 & 38.77 & 39.37 & 38.55 \\
& \textbf{SPIRE}
& \textbf{51.30} & \textbf{55.32} & \textbf{45.10} & \textbf{52.17} & \textbf{50.28} \\

\midrule

\multirow{4}{*}{ogbn-arxiv}
& FedAvg
& 18.56 & 14.58 & 15.56 & 14.33 & 13.15 \\
& GHOST
& 18.75 & 15.53 & 15.12 & 14.79 & 14.10 \\
& OASIS
& \conf{OOM} & \conf{OOM} & \conf{OOM} & \conf{OOM} & \conf{OOM} \\
& \textbf{SPIRE}
& \textbf{23.15} & \textbf{21.04} & \textbf{21.42} & \textbf{20.80} & \textbf{18.58} \\

\bottomrule
\end{tabular}
\end{table}

\subsection{Sensitivity Analysis}
\label{sec:sensitivity_analysis}
Figure~\ref{fig:sensitivity} reports accuracy on 
Cora under varying $\lambda_{se}$ and $\lambda_{align}$.

SPIRE shows strong robustness to $\lambda_{se}$: 
across four orders of magnitude from $10^{-7}$ to 
$10^{-3}$, accuracy fluctuates only between 
$62.80\%$ and $64.06\%$, indicating that structural 
regularization is beneficial but not brittle. This 
is consistent with its role as a supporting 
generative strategy rather than a primary objective.
The alignment weight $\lambda_{align}$ is more 
consequential. Setting it to $0.001$ causes accuracy 
to drop sharply to $44.50\%$, as insufficient 
alignment allows generated features to drift from 
the target prototypes. Performance recovers 
dramatically as $\lambda_{align}$ increases, 
reaching $64.06\%$ at $\lambda_{align}=0.1$ and 
remaining stable in $[0.1, 1.0]$. This confirms 
that $\mathcal{L}_{align}$, which closes the loop 
with the Wasserstein Barycenter target 
$\mathcal{B}^*$, is the critical constraint 
governing generation quality.

\begin{figure}[t]
	\centering
	% === 左图：lambda_se ===
	\begin{minipage}[b]{0.48\linewidth}
		\centering
		\includegraphics[width=\linewidth]{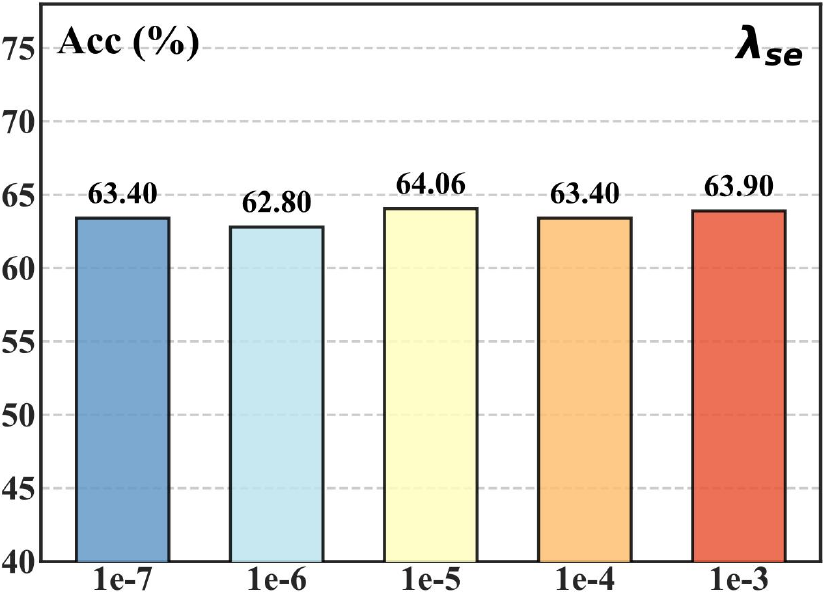}
		\centerline{\small (a) Impact of $\lambda_{se}$}
	\end{minipage}
	\hfill % 弹性间距，将两图推向两边
	% === 右图：lambda_align ===
	\begin{minipage}[b]{0.48\linewidth}
		\centering
		\includegraphics[width=\linewidth]{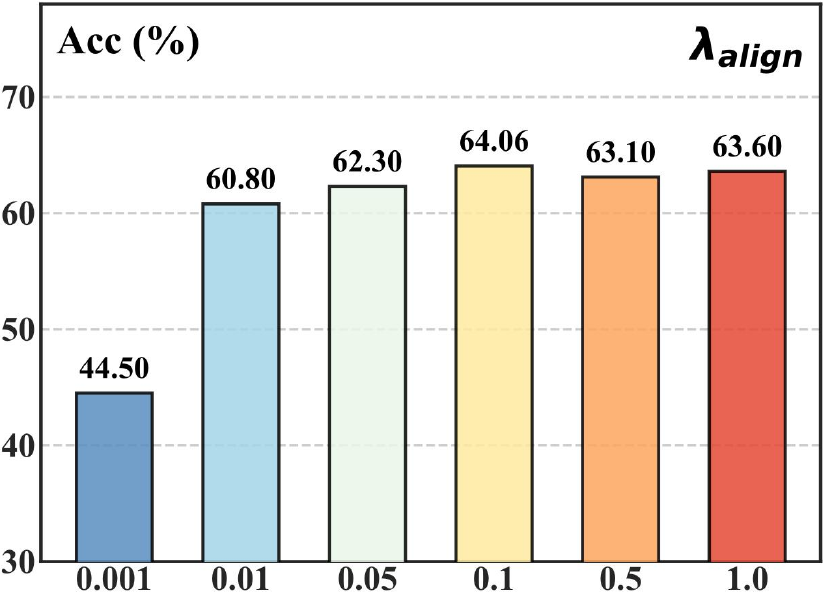}
		\centerline{\small (b) Impact of $\lambda_{align}$}
	\end{minipage}
	\caption{Hyperparameter Analysis on $\lambda_{se}$ and $\lambda_{align}$ on Cora.}
	\label{fig:sensitivity}
    \vspace*{-0.2cm}
\end{figure}

\begin{table*}[h]
\centering
\caption{Ablation study of SPIRE across seven datasets. Panel (a) evaluates the foundational impact of the core components. Panel (b) assesses the internal generative refinements via removal tests. Green and red subscripts denote the performance changes relative to the \textit{Baseline} and the \textit{Full Model}, respectively.}
\label{tab:ablation}
\resizebox{\textwidth}{!}{%
\begin{tabular}{l || cccc || ccc }
\toprule
\multirow{3}{*}{\textbf{Datasets}} & \multicolumn{4}{c||}{\textbf{(a) Core Architecture Ablation}} & \multicolumn{3}{c}{\textbf{(b) Generative Refinements}} \\
\cmidrule{2-8}
& \multirow{2}{*}{Baseline} & \textbf{+ SE} & \textbf{+ GFM} & \textbf{SPIRE} & \textbf{w/o WB} & \textbf{w/o Jitter} & \textbf{w/o $\boldsymbol{\gamma(s)}$} \\
& & (w/o GFM) & (w/o SE) & (Full Model) & (SE+GFM+Jit+$\gamma$) & (SE+GFM+WB+$\gamma$) & (SE+GFM+WB+Jit) \\
\midrule
\textbf{Cora}        & 39.20 & 44.20\up{5.00} & 59.80\up{20.60} & \cellcolor{lightyellow}\textbf{64.06}\up{24.86} & 59.40\down{4.66} & 62.33\down{1.73} & 63.79\down{0.27} \\
\textbf{CiteSeer}    & 29.90 & 41.90\up{12.00} & 47.90\up{18.00} & \cellcolor{lightyellow}\textbf{55.32}\up{25.42} & 50.50\down{4.82} & 53.75\down{1.57} & 52.50\down{2.82} \\
\textbf{PubMed}      & 63.20 & 67.30\up{4.10} & 69.70\up{6.50} & \cellcolor{lightyellow}\textbf{73.40}\up{10.20} & 70.30\down{3.10} & 71.90\down{1.50} & 72.15\down{1.25} \\
\textbf{Amz-Comp}    & 34.75 & 35.90\up{1.15} & 37.30\up{2.55} & \cellcolor{lightyellow}\textbf{38.32}\up{3.57} & 37.70\down{0.62} & 36.79\down{1.53} & 36.50\down{1.82} \\
\textbf{WikiCS}      & 18.50 & 19.32\up{0.82} & 20.04\up{1.54} & \cellcolor{lightyellow}\textbf{21.06}\up{2.56} & 20.78\down{0.28} & 19.98\down{1.08} & 20.50\down{0.56} \\
\textbf{Coauthor-CS} & 21.75 & 22.30\up{0.55} & 24.50\up{2.75} & \cellcolor{lightyellow}\textbf{29.65}\up{7.90} & 26.73\down{2.92} & 27.50\down{2.15} & 27.80\down{1.85} \\
\textbf{ogbn-arxiv}  & 9.57  & 10.11\up{0.54} & 16.30\up{6.73} & \cellcolor{lightyellow}\textbf{21.04}\up{11.47} & 19.58\down{1.46} & 20.75\down{0.29} & 20.30\down{0.74} \\
\bottomrule
\end{tabular}%
}
\end{table*}

\subsection{Ablation Study}
Table~\ref{tab:ablation} reflects the architectural hierarchy of SPIRE.

Panel~(a) isolates the core components, revealing their complementary roles across diverse graph structures. The Graph Diffusion Model (GFM, ablated via an MLP generator in \textit{w/o GFM}) acts as the feature reconstruction engine. On datasets like Cora, where node features strongly correlate with labels, GFM dominates by regenerating robust feature spaces (+20.60\%). However, feature reconstruction alone fails on sparse, highly fragmented topologies. On CiteSeer, known for isolated nodes and topological noise, Structural Entropy (SE) becomes indispensable. By replacing conventional volume-based weighting (\textit{w/o SE}) to filter noisy local structures and guide generation, SE independently yields a massive +12.00\% gain. Combining both achieves 64.06\% on Cora and 55.32\% on CiteSeer, confirming their necessary synergy: GFM executes the feature recovery, while SE enforces structural preference, preventing the generator from blindly overfitting to sparse non-IID topologies.

Panel~(b) examines internal generative refinements. Removing Wasserstein Barycenter initialization (\textit{w/o WB}), Jitter, and structural regularization (\textit{w/o $\gamma(s)$}) drops Cora accuracy by 4.66\%, 1.73\%, and 0.27\%, respectively. This progressive quantitative gap confirms the intended hierarchy. While SE and GFM act as foundational pillars addressing the volume-structure mismatch, WB, Jitter, and $\gamma(s)$ function strictly as internal stabilizers to fine-tune alignment and prevent deterministic overfitting.

\subsection{Discussion}

\begin{table}[t]
\centering
\vspace{-2pt}
\caption{Controlled stress test on Cora ($\text{Dirichlet } \beta=0.05, K=10$, 1 client corrupted) on global accuracy (\%).}
\label{tab:stress_test}
\vspace{-6pt}
\small
\begin{tabular}{lccc}
\toprule
\textbf{Injected Topology (1/10)} & \textbf{Volume-W} & \textbf{SPIRE (Ours)} & \textbf{Gain ($\Delta$)} \\
\midrule
Star Graph & 46.40 & \textbf{48.87} & \textbf{+2.47\%} \\
Degree-Rewired & 46.43 & \textbf{49.90} & \textbf{+3.47\%} \\
Random ER & 46.25 & \textbf{48.91} & \textbf{+2.66\%} \\
\bottomrule
\end{tabular}
\vspace{-4pt}
\end{table}

\textbf{Scope and Applicability of the Structural Descriptor.}
SPIRE employs normalized first-order structural entropy $S_k$ as a
compact degree-mass descriptor. Inherently, degree-level statistics
cannot resolve non-isomorphic graphs with identical degree sequences
(e.g., degree-preserving rewirings), nor determine whether highly
centralized connectivity (e.g., star graphs with artificially low
$S_k$) is beneficial for downstream learning. However, this theoretical
ambiguity does not cause catastrophic system failure in practice. In
controlled stress tests on Cora ($\text{Dirichlet }\beta=0.05, K=10$)
where one client is assigned a pathological topology
(Tab.~\ref{tab:stress_test}), SPIRE still outperforms volume-based
weighting by $+2.47\%$--$+3.47\%$ across all tested cases. Because
clients communicate feature prototypes rather than raw local edges,
server-side graph reconstruction limits the direct propagation of
local structural corruption to the global model. We therefore position
$S_k$ as an efficient $\mathcal{O}(1)$ inductive prior for differentiating
client influence under degree heterogeneity, rather than as a universal
graph-quality oracle.

\textbf{Privacy via Data Minimization.}
SPIRE enforces strict \textit{data minimization}: clients transmit only scalar summaries ($S_k, N_k$) and class prototypes $\bm{\mathcal{P}}_k$, keeping raw topologies ($\mathbf{A}_k$), node features ($\mathbf{X}_k$), and model parameters entirely local. Transmitting aggregated statistics avoids exposing sample-level data or gradient information. While SPIRE does not provide formal differential privacy (DP) guarantees against distribution leakage, its compact statistic interface is naturally compatible with standard DP perturbation mechanisms, which we leave for future privacy--utility studies.

\section{Conclusion}
In this paper, we presented SPIRE, a Structural Entropy-Driven Graph
Diffusion Generation method for one-shot federated graph learning, which addresses the volume-structure mismatch in client contribution estimation. By employing first-order structural entropy, SPIRE formulates client weighting as a principled entropy-regularized optimization and propagates this topology-aware bias into server-side pseudograph synthesis. Extensive evaluations across seven datasets demonstrate that SPIRE achieves superior classification performance and robust scalability, particularly on large graphs where conventional methods suffer from severe communication or memory bottlenecks. Furthermore, our theoretical analysis provides explicit bounds on the propagation of local entropy perturbations to the global barycentric initialization. While first-order entropy serves as a highly effective degree-level proxy without requiring raw topology exchange, future work will explore richer structural representations and formal privacy mechanisms for broader large-scale graph mining applications.

\section*{Ethical Considerations}
This work does not involve human subjects, personally identifiable information, or the collection of private user data. All experiments rely on publicly available standard graph benchmarks and are conducted in simulation. The proposed framework is intended for communication-efficient and privacy-conscious distributed graph mining, and its practical application should adhere to the privacy, security, and governance standards of the deployment domain.



\bibliographystyle{ACM-Reference-Format}
\bibliography{sample-base}

\end{document}